\documentclass[11pt,letterpaper]{article}
\usepackage[margin=1in]{geometry}
\usepackage{hyperref}
\hypersetup{
     colorlinks   = true,
     linkcolor = black,
     urlcolor    = teal,
	 citecolor = teal
}
\usepackage[margin=1in]{geometry}
\usepackage{times}
\usepackage[dvipsnames]{xcolor}
\usepackage{tikz}
\usepackage{framed}
\usepackage{soul}
\usepackage{placeins}

\usepackage{natbib}
\usepackage[mathscr]{eucal}

\usepackage{framed}
\usepackage{amsmath,amsfonts,amsbsy,amsgen,amscd,mathrsfs,amssymb,wrapfig,dsfont,authblk}
\usepackage[colorinlistoftodos]{todonotes}
\usepackage[utf8]{inputenc} %
\usepackage[T1]{fontenc}    %

\usepackage[nameinlink]{cleveref}
\usepackage{booktabs}       %
\usepackage{nicefrac}       %
\usepackage{microtype}      %
\usepackage[framemethod=default]{mdframed}

\usepackage{bm} 
\usepackage{bbm}
\usepackage{xspace}
\usepackage{algorithm}
\usepackage{algorithmic}
\usepackage{comment}
\usepackage{tcolorbox}

\usepackage{amsthm,thmtools,mathtools}
\usepackage[shortlabels]{enumitem}
\makeatletter
\def\set@curr@file#1{\def\@curr@file{#1}} %
\makeatother
\usepackage[load-configurations=version-1]{siunitx} %

\newtheorem{thm}{Theorem}
\newtheorem{lem}[thm]{Lemma}
\newtheorem{rem}{Remark}

\newtheorem{cor}[thm]{Corollary}

\newtheorem{exmpl}{Example}
\newtheorem{defn}{Definition}
\newtheorem*{defn*}{Definition}

\usepackage{prettyref}
\newrefformat{fig}{Figure~\ref{#1}}
\newrefformat{alg}{Algorithm~\ref{#1}}
\newrefformat{def}{Definition~\ref{#1}}
\newrefformat{eqn}{Equation~\ref{#1}}
\newrefformat{cor}{Corollary~\ref{#1}}
\newrefformat{app}{Appendix~\ref{#1}}
\newrefformat{rem}{Remark~\ref{#1}}
\newrefformat{assu}{Assumption~\ref{#1}}
\newrefformat{q}{Question~\ref{#1}}
\newrefformat{clm}{Claim~\ref{#1}}
\newrefformat{thm}{Theorem~\ref{#1}}
\newrefformat{exmpl}{Example~\ref{#1}}

\makeatletter
\newcounter{modules}

\makeatother

\newcommand{\inparen}[1]{\left ( #1 \right )}
\newcommand{\insquare}[1]{\left [ #1 \right ]}

\newcommand{\abs}[1]{\left\lvert #1 \right\rvert}

\newlength{\dhatheight}

\DeclareMathOperator*{\argmin}{argmin}

\DeclareMathOperator*{\Ex}{\mathbb{E}}
\DeclareMathOperator*{\E}{\mathbb{E}}

\DeclareMathOperator*{\Prob}{Pr}

\newcommand{\eps}{\varepsilon}

\newcommand{\calA}{\mathcal{A}}

\newcommand{\calD}{\mathcal{D}}

\newcommand{\calF}{\mathcal{F}}
\newcommand{\calG}{\mathcal{G}}
\newcommand{\calH}{\mathcal{H}}
\newcommand{\calI}{\mathcal{I}}

\newcommand{\calL}{\mathcal{L}}

\newcommand{\calO}{\mathcal{O}}

\newcommand{\calU}{\mathcal{U}}

\newcommand{\calX}{\mathcal{X}}
\newcommand{\calY}{\mathcal{Y}}

\newcommand{\ERM}{\textsf{ERM}\xspace}

\newcommand{\ind}{\mathbbm{1}}

\newcommand{\vc}{{\rm vc}}

\newcommand{\MAJ}{{\rm MAJ}}

\newcommand{\OPT}{\mathsf{OPT}}
\newcommand{\OPTDMax}{\mathsf{OPT}^{\mathcal{D}}_{\max}}
\newcommand{\OPTSMax}{\mathsf{OPT}^S_{\max}}

\renewcommand{\vec}[1]{\mathbf{#1}}

\newcommand{\RLoss}{\ell^{\text{rob}}}

\newcommand{\removed}[1]{}

\usepackage{float}
\usepackage{comment}
\usepackage[suppress]{color-edits}
\addauthor{sa}{teal}
\addauthor{kms}{blue}
\addauthor{ab}{purple}
\addauthor{om}{red}
\begin{document}
\title{Agnostic Multi-Robust Learning Using ERM\footnote{Authors are ordered alphabetically.}}
\author{Saba Ahmadi$^\dagger$}
\author{Avrim Blum$^\dagger$}
\author{Omar Montasser$^\ddagger$} 
\author{Kevin Stangl$^\dagger$}

\affil{$^\dagger$Toyota Technological Institute at Chicago\\
$^\ddagger$University of California, Berkeley\\
{\small\texttt{\{saba,avrim,omar,kevin\}@ttic.edu}}}

\maketitle

\begin{abstract}
A fundamental problem in robust learning is asymmetry: a learner needs to correctly classify every one of exponentially-many perturbations that an adversary might make to a test-time natural example. In contrast, the attacker only needs to find one successful perturbation. ~\citet{xiang2022patchcleanser} proposed an algorithm that in the context of patch attacks for image classification, reduces the effective number of perturbations from an exponential to a polynomial number of perturbations and learns using an ERM oracle. However, to achieve its guarantee, their algorithm requires the natural examples to be robustly realizable. 
This prompts the natural question; can we extend their approach to the non-robustly-realizable case where there is no classifier with zero robust error?

Our first contribution is to answer this question affirmatively by reducing this problem to a setting in which an algorithm proposed by ~\citet{DBLP:conf/colt/FeigeMS15} can be applied, and in the process extend their guarantees. Next, we extend our results to a multi-group setting and introduce a novel agnostic multi-robust learning problem where the goal is to learn a predictor that achieves low robust loss on a (potentially) rich collection of subgroups.
\end{abstract}

\section{Introduction}

Robustness to adversarial examples is considered a major contemporary challenge in machine learning. Adversarial examples are carefully crafted perturbations or manipulations of natural examples that cause machine learning predictors to miss-classify at test-time \citep{goodfellow2014explaining}. One particularly challenging aspect of this problem is the asymmetry between the learner and the adversary. Specifically, a learner needs to produce a predictor that is \textit{correct} on a randomly drawn natural example and \textit{robust} to potentially \textit{exponentially} many possible perturbations of it; while, the adversary needs to find just a single perturbation that fools the learner. In fact, because of this, adversarially robust learning has proven to require more sophisticated learning algorithms that go beyond standard Empirical Risk Minimization (ERM) in non-robust learning \citep{pmlr-v99-montasser19a}.

In patch attacks on images, for instance, an adversary can select one of an exponential number of designs for a patch to be placed in the image in order to cause a classification error. To address this exponential asymmetry between the learner and the adversary, recently \citet{xiang2022patchcleanser} introduced a clever algorithmic scheme, known as Patch-Cleanser, that provably reduces the exponential number of ways that an adversary can attack to a polynomial number of ways through the idea of masking images. 

Specifically, Patch-Cleanser's \emph{double-masking} approach is based on zero-ing out  two different contiguous blocks of an input image, hopefully to remove the adversarial patch. For each one-masked image, if for all possible locations of the second mask, the prediction model outputs the same classification, it means that the first mask removed the adversarial patch, and the agreed-upon prediction is correct. Any disagreements in these predictions imply that the mask was not covered by the first patch. 
\kmsmargincomment{Alt vsn: Crucially the Patch Cleanser requires that for a given image the classifier has zero error for every double masked image. In order to,... }Crucially, the Patch-Cleanser algorithm requires a {\em two-mask correctness} guarantee from an underlying predictor $F$ that is defined as follows: for a given input image $x$ and label $y$, if for any pair of masks applied to $x$, predictor $F$ outputs the correct prediction $y$, then $F$ has a two-correctness guarantee on $(x,y)$ \citep[see Definition 2 in][]{xiang2022patchcleanser}. In order to train a predictor with the two-mask correctness guarantee, \citet{xiang2022patchcleanser} augment the training dataset with pairs of masks at random locations of training images, and perform \emph{empirical risk minimization} ($\ERM$) on the augmented dataset.

\paragraph{Our Contributions}
When no predictor is perfectly correct on {\em all} perturbations (e.g., all two-mask operations), which we refer to as the 
the \textit{non-realizable} or \emph{agnostic} setting, we exhibit an example where plain $\ERM$ on the augmented dataset fails (See \prettyref{exmpl:ERM-failure}%
). At  a high-level, the main issue is that plain $\ERM$ on the augmented data-set treats all mistakes equally and so this could lead to learning a predictor with very high robust loss, i.e. on {\em many} training examples.
Our first contribution is to investigate whether the reduction proposed by~\citet{xiang2022patchcleanser} can be extended to the \emph{non-realizable} setting.
We answer this question affirmatively in~\prettyref{sec:non-realizable-oracle}, by building upon a prior work by~\citet{DBLP:conf/colt/FeigeMS15}.%

Next, in \prettyref{sec:multi-robustness}, we consider a multi-group setting and investigate the question of \emph{agnostic multi-robust learning} using an $\ERM$ oracle. This question is inspired by the literature on \emph{multi-calibration} and \emph{multi-group learning}~\citep{multicalib, multi-accuracy-Kim-etal,multiagnostic,multiagstrat, bugbounty}.
Our objective is that given a hypothesis class $\calH$ and a (potentially) rich collection of subgroups $\calG$, learn a predictor $h$ such that for each group $g\in\calG$, $h$ has low robust loss on $g\in \calG$. However, we highlight that the prior work on multi-group learning does not extend to the setting of robust loss since they do not consider adversarial perturbations of natural examples. To our knowledge, our work is the first to consider the notion of multi-group learning for robust loss. That being said we emphasize that there is a trade-off here; our guarantees are for the \emph{more challenging objective of robust loss}, but they are weaker than the ones given for PAC learning in the prior work. A detailed comparison is given in~\prettyref{sec:related-work}. 

Our motivation for studying multi-robustness is two-fold. First, to prohibit the adversary from targeting a specific demographic group
for adverse treatment. Additionally, it can increase the overall performance of the model by forcing the model to be robust on vulnerable examples. For instance, imagine a self-driving car system with a vision system recording a drive and we consider adversarial examples attacking individual frames of the video.
Ideally, the system would have robust performance over every frame. 
However, average robust error of $1\%$ could be very problematic if those errors instead of occurring uniformly then those errors concentrated on a specific adjacent set of frames.
In this %
example, imagine that the protected groups are nearby frames so that we maintain smooth and reliable performance \emph{locally and globally}.

To achieve multi-robustness, using plain $\ERM$ can fail by \emph{concentrating} the overall robust loss on a \emph{few} groups, instead of \emph{spreading} the loss across \emph{many} groups. However, building on our algorithm in \prettyref{sec:non-realizable-oracle} we propose~\prettyref{alg:boosting} that runs an additional layer of boosting with respect to groups to achieve multi-robustness guarantees across groups. We propose two types of multi-robustness guarantees, the first one is a randomized approach that guarantees the expected robust loss on each group is low (\prettyref{thm:generalization-multi-groups}). Next, we add a de-randomization step to derive deterministic guarantees for the robust loss incurred on each group (\prettyref{thm:generalization-multi-groups-deterministic}).

\subsection{Related Work}
\label{sec:related-work}
\paragraph{Patch Attacks} Patch attacks \citep{brown2017adversarial,karmon2018lavan,yang2020patchattack} 
are an important threat model in the general field of test-time evasion attacks \citep{goodfellow2014explaining}. 
Patch attacks realize adversarial test time evasion attacks to computer vision systems in the wild
by printing and attaching a patch to an object.
To mitigate this threat, 
there has been an active line of research for providing certifiable robustness guarantees against them \citep[see e.g.,][]{minorityreport, patchguard, patchguard++, bagcert, clippedbag,chiang2020}.

\paragraph{Adversarial Learning using $\ERM$}Recent work by \citep{DBLP:conf/colt/FeigeMS15} gives a reduction algorithm for adversarial learning using an ERM oracle, but their guarantee is only for finite hypothesis classes. We observe in this work that we can apply their reduction algorithm to our problem, and along the way, we extend the guarantees of their algorithm. A more detailed comparison is provided in \prettyref{sec:feigecomparison}.

\paragraph{Multi-group Learning}
Interestingly, the notion of multi-robustness has connections with a thriving area of work in algorithmic fairness centered on the notion of multi-calibration~\cite{multicalib, multi-accuracy-Kim-etal,multiagnostic,multiagstrat, bugbounty, gopalan2022loss}. %
The promise of these multi-guarantees, given a rich set of groups, is to ensure uniformly acceptable performance
on many groups simultaneously. 

Specifically, \citet{multiagnostic} show how to learn a predictor such that the loss experienced by every group is not much larger than the best
possible loss for this group within a given hypothesis class. However, we highlight that the prior work on multi-group learning does not extend to the setting of robust loss since their goal is not to minimize the robust loss by taking into consideration different perturbations of natural examples. In contrast, our approach can achieve multi-robustness guarantees by utilizing two layers of boosting to ensure `emphasis' on both specific groups and the adversarial perturbations.

\citet{multiagstrat,bugbounty} study the problem of minimizing a general loss function over a collection of subgroups. Their approach can capture the robust loss, however, the main distinction between their algorithm and our approach is that unlike them, we do not use group membership during the test time. This is essential when groups correspond to protected features, and therefore in some scenarios, it would be undesirable to incorporate them in decision models. Additionally, if we interpret some of the groups in our setting as objects to be classified like a stop-sign group or fire-hydrant group, %
then an approach that needs to detect group membership is too strong an assumption since the correct classification of those objects is our original goal.

However, we highlight that there is a trade-off here; To our knowledge, our work is the first one to achieve guarantees for the \emph{more challenging objective of robust learning without having access to the group membership of examples} but at the cost of achieving a weaker upper bound on the robust loss incurred on each group compared to the previous work on multi-group PAC learning. A detailed comparison is given in~\prettyref{sec:comparison-prior-work-multi-group-learning}.

\section{Setup and Notation} 
Let $\calX$ denote the instance space and $\calY$ denote the label space. Our main objective is to be robust against adversarial patches $\calA:\calX\to 2^\calX$, where $\calA(x)$ represents the (potentially infinite) set of adversarially patched images that an adversary might attack with at test-time on input $x$. \citet{xiang2022patchcleanser} showed that even though the space of adversarial patches $\calA(x)$ can be exponential or infinite, one can consider a ``covering'' %
function $\calU:\calX\to 2^\calX$ of masking operations on images where $\abs{\calU(x)}$ shows the covering set on input image $x$ and is polynomial in the image size. Thus, for the remainder of the paper, we focus on the task of learning a predictor robust to a perturbation set $\calU:\calX \to 2^{\calX}$, where %
$\calU(x)$is the set of allowed masking operations that can be performed on $x$. We assume that $\calU(x)$ is finite where $\abs{\calU(x)} \leq k$. 

We observe $m$ iid samples $S\sim \calD^m$ from an unknown distribution $\calD$, and our goal is to learn a predictor $\hat{h}$ achieving small robust risk:
\begin{equation}
\label{eqn:robrisk}
\Ex_{(x,y)\sim \calD} \insquare{ \max_{z\in \calU(x)} \ind[\hat{h}(z)\neq y]}.
\end{equation}

Let $\calH\subseteq \calY^\calX$ be a hypothesis class, and denote by $\vc(\calH)$ its VC dimension. Let $\ERM_\calH$ be an $\ERM$ oracle for $\calH$ that returns a hypothesis $h\in \calH$ that minimizes empirical loss. For any set arbitrary set $W$, denote by $\Delta(W)$ the set of distributions over $W$. 

In~\prettyref{sec:non-realizable-oracle}, we focus on a single-group setting where the benchmark $\OPT_{\calH}$ is defined as follows:
\begin{equation}
\label{eqn:opt}
    \OPT_{\calH} \triangleq \min_{h\in \calH} \Ex_{(x,y)\sim \calD} \max_{z \in \calU(x)} \ind\insquare{h(z)\neq y}. 
\end{equation}

In~\prettyref{sec:multi-robustness}, we consider a multi-group setting, where the instance space $\mathcal{X}$ is partitioned into a set of $g$ groups $\mathcal{G}=\{G_1,\dots,G_g\}$. These groups solely depend on the features $x$ and not the labels. The goal is to learn a predictor that has low robust loss on all the groups simultaneously. In this setup, the benchmark $\OPTDMax$ is as follows:
\begin{align}
\label{optmax}
   \OPTDMax =  \min_{h\in\calH} \max_{j\in[g]} \Ex_{(x,y)\sim D}\insquare{\max_{z\in \calU(x)} \ind[h(z)\neq y] \big| x\in G_j}
    \small
\end{align}

\section{Minimizing Robust Loss Using an ERM Oracle \label{erm}}
\label{sec:non-realizable-oracle}

First, we show an example where the approach of \citet{xiang2022patchcleanser} of calling $\ERM_\calH$ on the inflated dataset, i.e., original training points plus all possible perturbations resulting from the allowed masking operations, fails by obtaining %
a multiplicative gap of $k-1$ %
in the robust loss between the optimal robust classifier and the classifer returned by $\ERM_\calH$, where $k$ is the size of the perturbation sets. This gap exists since $\ERM$ can exhibit a solution that incorrectly classifies at least one perturbation per natural example, while there is a robust classifier that concentrates error on one natural example, thus getting low robust loss.

\begin{exmpl}
\label{exmpl:ERM-failure}
Consider the following example in $\mathcal{R}$. There is a training set $\{z_1,\cdots,z_{2n}\}$ of original examples, where examples $\{z_1,\cdots,z_n\}$ are positively labeled and are located at $x=1$. $\{z_{n+1},\cdots,z_{2n}\}$ are negatively labeled and are at $x=-1$. Each example $z_i$ has $k=n$ perturbations denoted by $\{z_{i,1},\cdots,z_{i,k}\}$. 

For each of the negative examples $\{z_{n+1},\cdots,z_{2n-1}\}$, all their perturbations are at $x=-0.75$. For the negative example $z_{2n}$, all its perturbations, i.e. $\{z_{2n,1},\cdots,z_{2n,k}\}$, are at $x=0$. For each positive example $z_i$ where $i\in\{1,\cdots,n-1\}$, one of their perturbations $z_{i,1}$ is at $x=0$ and the rest, i.e. $\{z_{i,2},\cdots,z_{i,k}\}$, are at $x=0.75$. For the positive example $z_n$, all its perturbations $z_{n,1},\cdots,z_{n,k}$ are at $x=0.75$.

 The adversarial training procedure considered in the paper by~\citet{xiang2022patchcleanser} runs ERM on the augmented dataset (original examples and all their perturbations) to minimize the 0/1 loss. ERM finds a threshold classifier $h_{ERM}$ with threshold $\tau=\eps_1$ for any $0<\eps_1<0.75$ that classifies any points with $x<\tau$ as negative and points with $x\geq \tau$ as positive. As a result, for each positive example $z_i$ for $i\in\{1,\cdots,n-1\}$, the perturbation $z_{i,1}$ is getting classified mistakenly which causes a robust loss on $z_i$. Therefore, $h_{ERM}$ has a robust loss of $(n-1)/2n$  since $n-1$ of the positive examples are not robustly classified. 
However, there exists a threshold classifier $h^*$ with threshold $\tau=\eps_2$ for any $-0.75<\eps_2<0$ that only makes mistakes on perturbations of $z_{2n}$ and thus has a robust loss of $1/2n$. However, its 0/1 loss is higher than $h_{ERM}$ and therefore ERM does not pick it.
Therefore, ERM can be suboptimal up to a multiplicative factor of $n-1$ for any arbitrary value of $n$. An illustration is given in~\prettyref{fig:example}.

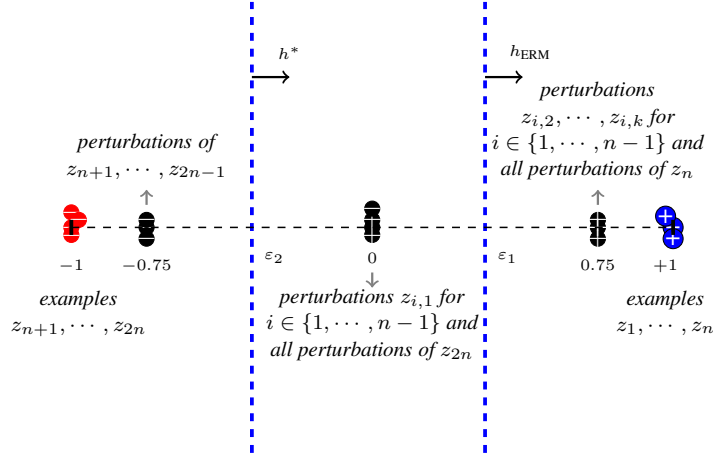
\begin{figure}[ht]
\label{fig:example}
\centering
\begin{tikzpicture}%

 \draw[red, fill=red,text=white] (-4,0) circle (0.1cm);
 \node[white] at (-4,0) {\small{$-$}};

 \draw[red, fill=red,text=white] (-4,0.2) circle (0.1cm);
 \node[white] at (-4,0.2) {\small{$-$}};

 \draw[red, fill=red,text=white] (-3.9,0.1) circle (0.1cm);
 \node[white] at (-3.9,0.1) {\small{$-$}};

 \draw[red, fill=red,text=white] (-4,-0.1) circle (0.1cm);
 \node[white] at (-4,-0.1) {\small{$-$}};
 
 \draw[line width=0.5mm] (-4,-0.1) -- (-4,0.1);
 \node[below] at (-4,-0.3) {\tiny{$-1$}};
 \node[below,align=center,font=\fontsize{8}{10}\itshape\selectfont] at (-3.9,-0.7) {examples\\ $z_{n+1},\cdots,z_{2n}$};
 
\draw[black, fill=black,text=white] (-3,0) circle (0.1cm);
\node[white] at (-3,0) {\small{$-$}};

 \draw[black, fill=black,text=white] (-3,0.1) circle (0.1cm);
\node[white] at (-3,0.1) {\small{$-$}};

 \draw[black, fill=black,text=white] (-3,-0.15) circle (0.1cm);
\node[white] at (-3,-0.15) {\small{$-$}};
\node[below] at (-3,-0.3) {\tiny{$-0.75$}};
\node[above,align=center,font=\fontsize{8}{10}\itshape\selectfont] at (-3,0.5) {perturbations of\\ $z_{n+1},\cdots,z_{2n-1}$};
\draw[->,line width=0.3mm,gray] (-3,0.3)--(-3,0.5);
\draw[black, fill=black,text=white] (+3,0) circle (0.1cm);
\node[white] at (+3,0) {\small{$+$}};

 \draw[black, fill=black,text=white] (+3,0.1) circle (0.1cm);
\node[white] at (+3,0.1) {\small{$+$}};

 \draw[black, fill=black,text=white] (+3,-0.15) circle (0.1cm);
\node[white] at (+3,-0.15) {\small{$+$}};
\node[below] at (+3,-0.3) {\tiny{$0.75$}};
\node[above,align=center,font=\fontsize{8}{10}\itshape\selectfont] at (+3,0.5) {perturbations \\ $z_{i,2},\cdots,z_{i,k}$ for\\ $i\in\{1,\cdots,n-1\}$ and\\ all perturbations of $z_n$};
\draw[->,line width=0.3mm,gray] (3,0.3)--(3,0.5);

\draw[black, fill=black,text=white] (0,0) circle (0.1cm);
\node[white] at (+3,0) {\small{$+$}};

 \draw[black, fill=black,text=white] (0,0.1) circle (0.1cm);
\node[white] at (0,0.1) {\small{$+$}};

 \draw[black, fill=black,text=white] (0,-0.1) circle (0.1cm);
\node[white] at (0,-0.1) {\small{$+$}};

 \draw[black, fill=black,text=white] (0,0.25) circle (0.1cm);
\node[white] at (0,0.25) {\small{$-$}};

\node[below] at (0,-0.2) {\tiny{$0$}};
\node[above,align=center,font=\fontsize{8}{10}\itshape\selectfont] at (0,-1.9) {perturbations $z_{i,1}$ for\\ $i\in\{1,\cdots,n-1\}$ and\\ all perturbations of $z_{2n}$};

\draw[->,line width=0.3mm,gray] (0,-0.6)--(0,-0.8);
\node[circle,draw, minimum size = 0.1cm, inner sep=0pt,fill=blue] (C) at  (+4,0) [text=white] {\small{+}};

\node[circle,draw, minimum size = 0.1cm, inner sep=0pt,fill=blue] (C) at  (+3.9,0.15) [text=white] {\small{+}};

\node[circle,draw, minimum size = 0.1cm, inner sep=0pt,fill=blue] (C) at  (+4,-0.15) [text=white] {\small{+}};
\node[below] at (3.9,-0.3) {\tiny{$+1$}};
\node[below,align=center,font=\fontsize{8}{10}\itshape\selectfont] at (3.9,-0.7) {examples\\ $z_{1},\cdots,z_{n}$};

\draw[line width=0.5mm] (4,-0.1) -- (4,0.1);

\draw[line width=0.5mm,blue,dashed] (1.5,-3) -- (1.5,3);
\node[below] at (1.8,-0.2) {\tiny{$\eps_1$}};

\draw[->,line width=0.3mm] (1.5,2) -- (2,2);
\node[above] at (2.1,2.1) {\tiny{$h_{\text{ERM}}$}};

\draw[line width=0.5mm,blue,dashed] (-1.6,-3) -- (-1.6,3);
\node[below] at (-1.3,-0.2) {\tiny{$\eps_2$}};

\draw[->,line width=0.3mm] (-1.6,2) -- (-1.1,2);
\node[above] at (-1.1,2.1) {\tiny{$h^*$}};

\draw[dashed, line width=0.2mm] (-4,0) -- (4,0);

\end{tikzpicture}
\caption{
$\ERM$ failure mode in the robustly un-realizable case. Blue, red, and black points show respectively original examples with a positive label, original examples with a negative label, and perturbations of original examples.
}
\end{figure}
\end{exmpl}

Next, we present our first contribution: %
we show in~\prettyref{thm:generalization-FMS} that~\prettyref{alg:FMS} proposed by 
~\citet{DBLP:conf/colt/FeigeMS15} learns a predictor that is simultaneously robust to a set of (polynomially many) masking operations, using an $\ERM_\calH$ oracle. The algorithm is based on prior work%
, but the analysis and application are novel in this work. A detailed comparison with~\citet{DBLP:conf/colt/FeigeMS15} is given in~\prettyref{sec:feigecomparison}. The main interesting feature of this algorithm is that it achieves stronger robustness guarantees in the non-realizable regime when $\OPT_\calH \gg 0$, where the approach of \citet{xiang2022patchcleanser} can fail as mentioned in~\prettyref{exmpl:ERM-failure}. 

\begin{algorithm}%
\begin{algorithmic}[1]
\caption{\citet*{DBLP:conf/colt/FeigeMS15}}
\label{alg:FMS}
  \INPUT weight update parameter $\eta>0$, number of rounds $T$, and training dataset $S=\{(x_1,y_1),\dots, (x_m,y_m)\}$ and corresponding weights $p_1,\cdots,p_m$\;
  
  \STATE 
  Set $w_1(z, (x,y)) = 1$, for each $(x,y)\in S, z\in \calU(x)$.\;
  
  \STATE 
  Set $P^1(z,(x,y)) = \frac{w_1(z,(x,y))}{\sum_{z'\in\calU(x)} w_1(z',(x,y))}$, for each $(x,y)\in S, z\in \calU(x)$.\;
  
\FOR{each $t\in \{1,\cdots T\}$}%
\STATE 
Call $\ERM$ on the empirical weighted distribution:\;
\STATE 
{\small\[h_t = \argmin_{h\in\mathcal{H}} \sum_{(x,y)\in S} \sum_{z\in \calU(x)} {p_{(x,y)} }P^t(z,(x,y)) \ind\insquare{h_{t}(z)\neq y}\]}\;
\FOR{each $(x,y)\in S$ and $z\in\calU(x)$}%
\STATE {\small $w_{t+1}(z,(x,y)) = (1+\eta \ind\insquare{h_{t}(z)\neq y}) \cdot w_{t}(z, (x,y))$}\;

\STATE 
$P^{t+1}(z,(x,y))=\frac{w_t(z,(x,y))}{\sum_{z'\in\calU(x)} w_t(z',(x,y))}$\;
\ENDFOR
\ENDFOR
\OUTPUT The majority-vote predictor $\MAJ(h_1,\dots, h_T)$. \\
\end{algorithmic}
\end{algorithm}

\begin{thm}
\label{thm:generalization-FMS}
Set $T(\eps) = \frac{32 \ln k}{\eps^2}$ and $m(\eps, \delta) = O\inparen{\frac{\vc(\calH)(\ln k)^2}{\eps^4}\ln \inparen{\frac{\ln k}{\eps^2}}+\frac{\ln(1/\delta)}{\eps^2}}$. Then, for any distribution $\calD$ over $\calX\times \calY$, with probability at least $1-\delta$ over $S\sim \calD^{m(\eps,\delta)}$, running \prettyref{alg:FMS} {where $p_{(x,y)}=1/m$ for all $(x,y)\in S$ }for $T(\eps)$ rounds produces $h_1,\dots,h_{T(\eps)}$ satisfying:

\small\[\Ex_{(x,y)\sim \calD} \insquare{ \max_{z\in \calU(x)} \ind\insquare{\MAJ(h_1,\dots, h_{T(\eps)})(z)\neq y} } 
\leq 2\OPT_{\calH} + \eps\]
where $\MAJ(h_1,\dots, h_{T(\eps)})$ shows the majority-vote of predictors $h_1,\dots, h_{T(\eps)}$.
\end{thm}

\begin{rem}
In the approach proposed by~\citet{xiang2022patchcleanser}, the robust loss with respect to the (exponentially many) patches is upper bounded by the robust loss with respect to the (polynomially many) masking operations. Therefore,~\prettyref{thm:generalization-FMS} implies that the robust loss against patches is at most $2\OPT_{\calH} + \eps$.
\end{rem}

\subsection{Comparison with prior related work} 
\label{sec:feigecomparison}
As presented, \cite{DBLP:conf/colt/FeigeMS15} only considered \emph{finite} hypothesis classes $\calH$ and provided generalization guarantees depending on $\log\abs{\calH}$. On the other hand, we consider here infinite classes $\calH$ with bounded VC dimension and provide tighter robust generalization bounds (see \prettyref{thm:generalization-FMS}). We would also like to highlight another difference. Given an output of $h_1,\dots,h_T$ from \prettyref{alg:FMS}, the guarantee provided by \cite{DBLP:conf/colt/FeigeMS15} is on average and does not exactly capture the notion of robust loss %
i.e. the loss on input $x$ is $\sup_{z\in\calU(x)} \frac{1}{T}\sum_{t=1}^{T} \ind[h_t(z)\neq y]$ (\prettyref{lem:FMS} states their result). We emphasize that this is different from the \emph{robust loss} guarantee that we obtain in~\prettyref{thm:generalization-FMS} for a %
single classifier, i.e. the loss on input 
$x$ is captured as $\sup_{z\in\calU(x)} \ind[\MAJ(h_1,\dots,h_T)(z)\neq y]$. %
In particular, unlike the %
guarantee provided by \cite{DBLP:conf/colt/FeigeMS15} in which the adversary chooses $z\in\calU(x)$ and then we can probabilistically choose a classifier to classify it, to implement the Patch-Cleanser reduction we need a single classifier that is \emph{simultaneously} correct on \emph{all} 
$z\in\calU(x)$. Because of the difference in guarantees derived, we incur a multiplicative factor of 2 compared with their bound. 

The robust learning guarantee \citep[][Theorem 2]{attias2022improved} assumes access to a \emph{robust} $\ERM$ oracle, which minimizes the robust loss on the training dataset. On the other hand, at the expense of higher sample complexity, we provide a robust learning guarantee using only an $\ERM$ oracle which is a more common and simpler assumption in the challenging \emph{non-realizable} setting. Prior work due to \citet{DBLP:conf/nips/MontasserHS20} considered using an $\ERM$ oracle for robust learning but only in the simpler realizable setting (when $\OPT_\calH=0$).

\subsection{Proof of \prettyref{thm:generalization-FMS}}
We exhibit a sketch for the proof of~\prettyref{thm:generalization-FMS} and defer the details to the Appendix. The main insight is to solve a finite zero-sum game. In particular, our goal is to find a mixed-strategy over the hypothesis class that is approximately close to the value of the game:
\[\OPT_{S,\calH} \triangleq \min_{h\in \calH} \frac{1}{m}\sum_{i=1}^{m} \max_{z_i\in \calU(x_i)} \ind\insquare{h(z_i)\neq y_i}.\]

We observe that \prettyref{alg:FMS} due to \citep{DBLP:conf/colt/FeigeMS15} solves a similar finite zero-sum game (see \prettyref{lem:FMS}), and then we relate it to the value of the game we are interested in (see \prettyref{lem:opt}). Combined together, this only establishes that we can minimize the robust loss on the empirical dataset using an $\ERM$ oracle. We then appeal to uniform convergence guarantees for the robust loss in \prettyref{lem:unif-robloss} to show that, with a large enough training data, our output predictor achieves robust risk that is close to the value of the game. 

\begin{lem}
\label{lem:opt}
For any dataset $S %
=\{(x_1,y_1),\dots,(x_m,y_m)\}\in (\calX\times\calY)^m$ {with corresponding weights $p_1,\cdots,p_m=1/m$},
\begin{align*}
    &\OPT_{S,\calH} = \min_{h\in \calH} \frac{1}{m}\sum_{i=1}^{m} \max_{z_i\in \calU(x_i)} \ind\insquare{h(z_i)\neq y_i} \\
    &\geq
    \min_{Q\in \Delta(\calH)} \max_{\substack{P_{1}\in \Delta(\calU(x_1)),\\ \dots\\ P_{m} \in \Delta(\calU(x_m))}} \frac{1}{m} \sum_{i=1}^{m} \Ex_{z_i\sim P_i } \Ex_{h\sim Q} \ind\insquare{h(z_i)\neq y_i}
\end{align*}
\end{lem}

\begin{lem} [\citet*{DBLP:conf/colt/FeigeMS15}]
\label{lem:FMS}
For any data set $S %
=\{(x_1,y_1),\dots,(x_m,y_m)\}\in (\calX\times\calY)^m$ {with corresponding weights $p_1,\cdots,p_m=1/m$}, running \prettyref{alg:FMS} for $T$ rounds produces a mixed-strategy $\hat{Q} = \frac{1}{T} \sum_{t=1}^{T} h_t \in \Delta(\calH)$ satisfying:
{\small
\begin{align*}
    &\max_{\substack{P_1\in \Delta(\calU(x_1)),\\ \dots,\\ P_m\in \Delta(\calU(x_m))}} \frac{1}{m}\sum_{i=1}^{m} \Ex_{z_i\sim P_i} \frac{1}{T} \sum_{t=1}^{T} \ind\insquare{h_t(z_i)\neq y_i} \\
    &\leq
    \min_{Q\in \Delta(\calH)} \max_{\substack{P_{1}\in \Delta(\calU(x_1)),\\ \dots,\\ P_{m} \in \Delta(\calU(x_m))}} \frac{1}{m} \sum_{i=1}^{m} \Ex_{z_i\sim P_i } \Ex_{h\sim Q} \ind\insquare{h(z_i)\neq y_i} + 
    2\sqrt{\frac{\ln k}{T}}
\end{align*}
}%
\end{lem}

\removed{
\begin{lem} [Extension to weighted samples]
For any data set $S %
=\{(x_1,y_1),\dots,(x_m,y_m)\}\in (\calX\times\calY)^m$ and any corresponding weights $p_1,\dots, p_m > 0$ such that $\sum_{i=1}^{m} p_i = 1$, running \prettyref{alg:FMS} for $T$ rounds produces a mixed-strategy $\hat{Q} = \frac{1}{T} \sum_{t=1}^{T} h_t \in \Delta(\calH)$ satisfying:
\begin{align*}
    \max_{P_1\in \Delta(\calU(x_1)),\dots,P_m\in \Delta(\calU(x_m))} &\sum_{i=1}^{m} p_i\cdot \Ex_{z_i\sim P_i} \frac{1}{T} \sum_{t=1}^{T} \ind\insquare{h_t(z_i)\neq y_i} \leq\\
    &\min_{Q\in \Delta(\calH)} \max_{P_{1}\in \Delta(\calU(x_1)),\dots, P_{m} \in \Delta(\calU(x_m))} \sum_{i=1}^{m} p_i\cdot \Ex_{z_i\sim P_i } \Ex_{h\sim Q} \ind\insquare{h(z_i)\neq y_i} + 2\sqrt{\frac{\ln k}{T}}.
\end{align*}
\label{lem:extension-FMS-weights}
\end{lem}
}

\begin{lem} [VC Dimension for the Robust Loss \citep{attias2022improved}]
\label{lem:unif-robloss}
For any class $\calH$ and any $\calU$ such that $\sup_{x\in\calX}\abs{\calU(x)}\leq k$, denote the robust loss class of $\calH$ with respect to $\calU$ by
\[\calL^{\calU}_{\calH} = \{(x,y)\mapsto \max_{z\in\calU(x)} \ind\insquare{h(z)\neq y}: h\in\calH\}.\]
Then, it holds that $\vc(\calL^{\calU}_{\calH})\leq O(\vc(\calH) \log(k))$. 
\end{lem}

\section{Multi-robustness guarantees on a set of groups}
\label{sec:multi-robustness}
In this section, we propose a boosting algorithm that learns a predictor with a low robust loss on a collection of subgroups simultaneously. 
First, we consider the case of disjoint groups and present our training-time algorithm for this case in~\prettyref{sec:multi-robust}. 
~\prettyref{sec:generalization-guarantees} provides generalization guarantees. In~\prettyref{sec:overlapping-groups-reduction}, we show a reduction from overlapping groups to disjoint groups. In the following, first we formalize the notions of robust loss on a specific group and multi-robustness.

When the training dataset $S$ is partitioned into $g$ groups $\calG=\{G_1,\dots,G_g\}$,
the empirical robust loss of a predictor $h$ on group $G_j$ is defined as follows:
\begin{align}
&\RLoss_j(h)=\frac{1}{|G_j|}\sum_{(x,y)\in G_j}\max_{z\in\mathcal{U}(x)}\ind[h(z)\neq y]\label{eqn:unweighted-robust-loss}
\end{align}

The learning benchmark that we compete with on a dataset $S$ for the robust loss on each group is $\OPT^{S}_{\max}$ that is defined as follows:
\begin{align}
&\OPT^{S}_{\max}=\min_{h\in \mathcal{H}}\max_{j\in[g]}\frac{1}{|G_j|}\sum_{(x,y)\in G_j}\max_{z\in\mathcal{U}(x)}\ind[h(z)\neq y]
\label{defn: optmax}
\end{align}

\begin{defn}[Multi-Robustness]
\label{defn:multirob}
A hypothesis $h$ is multi-robust on a dataset $S$ if it achieves the following guarantee:
\begin{align*}
&\max_{j\in[g]}\frac{1}{|G_j|}\sum_{(x,y)\in G_j}\max_{z\in\mathcal{U}(x)}\ind[h(z)\neq y]\leq \OPT^S_{\max}+\eps 
\end{align*}
\label{def:multi-robustness}
\end{defn}

\begin{defn}[$\beta$-Multi-Robustness]
A hypothesis $h$ is $\beta$-multi-robust on a dataset $S$ if it achieves the following guarantee:
\begin{align*}
&\max_{j\in[g]}\frac{1}{|G_j|}\sum_{(x,y)\in G_j}\max_{z\in\mathcal{U}(x)}\ind[h(z)\neq y]\leq \beta(\OPT^{S}_{\max}+\eps) \label{eq:beta-multi-robustness}
\end{align*}
\label{defn:beta-multi-robustness}
\end{defn}

\begin{defn}[Multi-Robustness on Average]
A set of hypotheses $\mathcal{H'}=\{h_1,\dots,h_T\}$ is multi-robust on a dataset $S$ on average if the the following property holds:
\[\frac{1}{T}\max_{j\in [g]}\sum_{t=1}^T\RLoss_j(h_t)\leq \OPT^{S}_{\max}+\eps\]
\label{def:avg-multi-robustness}
\end{defn}

\begin{rem}
~\prettyref{def:multi-robustness} is a stronger notion of multi-robustness compared to~\prettyref{def:avg-multi-robustness}.
\end{rem}

\paragraph{Summary of Results.}
\prettyref{sec:multi-robust} investigates the case of disjoint groups and proposes a two-layer boosting algorithm (\prettyref{alg:boosting}) that achieves multi-robustness on the training dataset $S$. First, we show that $\calH'=\{h_1,\dots,h_T\}$ returned by~\prettyref{alg:boosting} is multi-robust on average (\prettyref{thm:randomized-multi-robustness}).~\prettyref{thm:deterministic-multi-robustness} exhibits that the majority-vote classifier over $\calH'$, i.e. $\MAJ(h_1,\dots,h_T)$, obtains $\beta$-multi-robustness for $\beta=2$. We remark that although~\prettyref{thm:randomized-multi-robustness} achieves a tighter upper bound on the multi-robustness guarantee,~\prettyref{thm:deterministic-multi-robustness} gives a guarantee for the stronger notion of multi-robustness.~\prettyref{sec:generalization-guarantees} provides generalization guarantees for both notions of average multi-robustness and $\beta$-multi-robustness. In~\prettyref{sec:overlapping-groups-reduction}, we show a reduction from overlapping groups to disjoint groups.

\subsection{Comparison to Prior Work on Multi-group Learning}
\label{sec:comparison-prior-work-multi-group-learning}
\citet{multiagnostic} study agnostic multi-group \emph{PAC learning} and guarantee that for each group $G_j$ in a collection of groups $\calG$:
\[\Ex\insquare{\ell(h(x),y)|x\in G_j}\leq \min_{h_{G_j}\in\calH}\Ex\insquare{\ell(h_{G_j}(x),y)|x\in G_j}\]
That is, the hypothesis $h$ must compete against a hypothesis $h_{G_j}\in \calH$
 trained specifically to minimize the error over the group $G_j\in\calG$, for every group in the collection. \emph{However, their results do not extend to the case of robust loss.} In contrast, in our notion of multi-robustness loss that holds for the \emph{more challenging objective of robust learning}, our benchmark is weaker (~\prettyref{def:multi-robustness}). 
We leave it as an open question to study whether our upper bounds for the robust loss over a collection of groups can be strengthened.

\subsection{Boosting algorithm achieving multi-robustness guarantees:}
\label{sec:multi-robust}

In this section, we present~\prettyref{alg:boosting} that obtains %
multi-robustness guarantees on a set of disjoint groups. The algorithm follows the idea proposed by~\citet{freund1996game} that obtains boosting by playing a repeated game. Initially a sample set $S=\{(x_1,y_1),\dots,(x_m,y_m)\}$ partitioned into a set of disjoint groups $\mathcal{G}=\{G_1,\dots,G_{g}\}$ is received as input. %
$P_j^t$ shows the normalized weight of group $G_j$ in step $t$. Initially, for each group $G_j$, $P^t_j=1/g$.
In each round $t$, the weight of each group gets split between its examples equally: $p_i = P^t_j/|G_j|$ where $(x_i,y_i)\in G_j$. Subsequently, an oracle call is made to{~\prettyref{alg:FMS}} with sample weights $p_1,\dots,p_m$. %
~\prettyref{lem:avg-robust-loss-upper-bound} shows at each iteration $t$,{~\prettyref{alg:FMS}} returns a hypothesis $h_t$ such that its average robust loss across the groups is at most $\OPT^{S}_{\max}+\eps$. In the next iteration $t+1$, for each group $G_j$, the weights of examples in $G_j$ get decreased by a multiplicative factor of 
$1-\delta m_j^{\text{rob}}(h_t)$ where $m_j^{\text{rob}}(h_t)=1-\RLoss_j(h_t)$ and $\delta=\sqrt{{\ln g}/{T}}$.~\prettyref{thm:randomized-multi-robustness} exhibits that after %
$T=\calO({\ln g}/{\eps^2})$ rounds,~\prettyref{alg:boosting} outputs
a set of hypotheses $\mathcal{H'}=\{h_1,\dots,h_T\}$ such that for each group $G_j$ the average multi-robustness guarantee is obtained, i.e., $\frac{1}{T}\sum_{t=1}^T \RLoss_j(h_t)\leq \OPT^S_{\max}+\eps$.~\prettyref{thm:deterministic-multi-robustness} provides that $\MAJ(h_1,\dots,h_t)$ achieves
$\beta$-multi-robustness guarantee for $\beta=2$.

\begin{algorithm}[H]
\caption{Boosting Algorithm Achieving Multi-Robustness}
\label{alg:boosting}
\begin{algorithmic}
    \INPUT training dataset $S=\{(x_1,y_1),\dots,(x_m,y_m)\}$ partitioned into a set of groups $\{G_1,\cdots,G_g\}$\;
    
    Initially, $\forall 1\leq j\leq g: P_j^t = 1/g$\;
    
    \FOR{$t=1,\dots,T$}
    \STATE $p_i = P^t_j/|G_j|$ where $(x_i,y_i)\in G_j$\;
    
    \STATE 
    Call{~\prettyref{alg:FMS}} on $S$ with weights $(p_1,\dots,p_m)$ for $T'=\frac{36\ln k}{\eps^2}$ rounds.\;
    
    \STATE 
    Update $P^t_j,  \text{ for all }j\in[g]$:\;
    
    \STATE 
    \[ P^{t+1}_j=\frac{P^t_j\cdot\left(1-\delta m_j^{\text{rob}}(h_t)\right)}{Z_t}\]
    where $m_j^{\text{rob}}(h_t)=1-\RLoss_j(h_t)$, $Z_t$ is a normalization factor, and     $\delta=\sqrt{\frac{\ln g}{T}}$.\;
    \ENDFOR
    \OUTPUT 
    $\mathcal{H'}=\{h_1,\cdots,h_T\}$
\end{algorithmic}
\end{algorithm}

\begin{rem}
We remark that the output of~\prettyref{alg:boosting} is a set of majority-vote classifiers over $\calH$:
\begin{align*}
&\calH'=\Big\{\MAJ(h_{1,1},\dots, h_{1,T'}),\dots,\MAJ(h_{T,1},\dots, h_{T,T'}):\forall i\in[T], \forall j\in[T'], h_{i,j}\in \calH\Big\}
\end{align*}
\end{rem}

Before proving the multi-robustness guarantees, we show that~\prettyref{lem:avg-robust-loss-upper-bound} holds. {In order to prove that~\prettyref{lem:avg-robust-loss-upper-bound} holds, first we show in ~\prettyref{lem:extension-FMS-weights} that an extension of~\prettyref{lem:FMS} holds when $p_1,\cdots,p_m$ are arbitrary weights such that $\sum_{i=1}^m p_i = 1$.} Next, we restate the guarantee of the Multiplicative Weights algorithm that is a generalization of \emph{Weighted Majority} algorithm~\citep{littlestone1994weighted} and is equivalent to \emph{Hedge} developed by~\citet{freund1997decision}.

\begin{lem} [Extension to general weights]
For any dataset $S =\{(x_1,y_1),\dots,(x_m,y_m)\}\in (\calX\times\calY)^m$ and any corresponding weights $p_1,\dots, p_m > 0$ such that $\sum_{i=1}^{m} p_i = 1$, running
{\prettyref{alg:FMS}}
for $T$ rounds produces a mixed-strategy $\hat{Q} = \frac{1}{T} \sum_{t=1}^{T} h_t \in \Delta(\calH)$ satisfying:
\begin{align*}
   &\max_{\substack{P_1\in \Delta(\calU(x_1)),\\ \dots,\\P_m\in \Delta(\calU(x_m))}} \sum_{i=1}^{m} p_i\cdot \Ex_{z_i\sim P_i} \frac{1}{T} \sum_{t=1}^{T} \ind\insquare{h_t(z_i)\neq y_i} \\
   &\leq
    \min_{Q\in \Delta(\calH)} \max_{\substack{P_{1}\in \Delta(\calU(x_1)),\\ \dots, \\P_{m} \in \Delta(\calU(x_m))}} \sum_{i=1}^{m} p_i\cdot \Ex_{z_i\sim P_i } \Ex_{h\sim Q} \ind\insquare{h(z_i)\neq y_i} 
    + 2\sqrt{\frac{\ln k}{T}}
\end{align*}
\label{lem:extension-FMS-weights}
\end{lem}
\begin{lem}
In each round $t$ of~\prettyref{alg:boosting}, by making an oracle-call to{~\prettyref{alg:FMS}} after $T'=\frac{4\ln k}{\eps^2}$ rounds, %
a hypothesis $h_t$ is outputted such that $\E_{j\sim P^t}[\ell^{rob}_j(h_t)]=\sum_{j\in[g]}P_j^{t}\ell_j^{rob}(h_t)\leq \OPT^S_{\max}+\eps$.
\label{lem:avg-robust-loss-upper-bound}
\end{lem}

\begin{thm}[Mutiplicative Weights Algorithm \citep{kale2007efficient}]
\label{thm:MW_alg}
For any sequence of costs of experts $\vec{m}_1,\cdots,\vec{m}_T$ revealed by nature where all the costs are in $[0,1]$, the sequence of mixed strategies $\vec{p}_1,\cdots,\vec{p}_T$ produced by the Multiplicative Weights algorithm satisfies:
\[\sum_{t=1}^T \vec{m}_t\cdot \vec{p}_t\leq (1+\delta)\min_{\vec{p}}\sum_{t=1}^T\vec{m}_t\cdot \vec{p}+\frac{\ln n}{\delta}\]
where $n$ is the number of experts.
\end{thm}

\begin{thm}
\label{thm:avg-empirical-boosting}
When $T=\calO(\frac{\ln g}{\eps^2})$,~\prettyref{alg:boosting} computes a set of hypotheses $\mathcal{H}'=\{h_1,\cdots,h_T\}$, such that for each group $G_j$, 
$\frac{1}{T}\sum_{t=1}^T\RLoss_j(h_t)\leq \OPT^S_{\max}+\eps$.
\label{thm:randomized-multi-robustness}
\end{thm}
\begin{proof}

In each iteration $t$, we define average loss and reward terms as follows:

\[L(P^t,h_t)=\E_{j\sim P_t}\Big[\RLoss_j(h_t)\Big]=\sum_{j\in[g]}P^t_j\RLoss_j(h_t),\] 
\[M(P^t,h_t)=\E_{j\sim P_t}\Big[m_j^{\text{rob}}(h_t)\Big]\]
Substituting $\RLoss_j(h_t)=1-m_j^{\text{rob}}(h_t)$ provides:
\begin{align*}
&M(P^t,h_t)=\sum_{j\in[g]}P^t_j(1-\RLoss_j(h_t))=1-\sum_{j\in[g]}P^t_j\RLoss_j(h_t)\\
&=1-L(P^t,h_t)
\end{align*}
Now by setting $T=\frac{9\ln g}{\eps^2}$ which implies that $\delta=\sqrt{\frac{\ln g}{T}}=\frac{\eps}{3}$, and by using the guarantee of~\prettyref{thm:MW_alg}, the following bound is obtained.
\begin{align*}
&\frac{1}{T}\sum_{t=1}^T M(P^t,h_t)\leq \frac{(1+\delta)}{T}\min_{j\in[g]}\sum_{t=1}^T M(j,h_t)+\frac{\ln g}{\delta T}\\
&\rightarrow \frac{1}{T}\sum_{t=1}^T M(P^t,h_t)\leq \frac{1}{T}\min_{j\in[g]}\sum_{t=1}^T M(j,h_t)+\delta+\frac{\ln g}{\delta T}\\
&\rightarrow \frac{1}{T}\sum_{t=1}^T M(P^t,h_t)\leq \frac{1}{T}\min_{j\in[g]}\sum_{t=1}^T M(j,h_t)+\frac{2\eps}{3}
\end{align*}
where $M(j,h_t)$ is the reward term when the whole probability mass is concentrated on group $G_j$. Therefore for each group $j\in[g]$:
\begin{align}
&\frac{1}{T}\sum_{t=1}^T M(j,h_t)\geq \frac{1}{T}\sum_{t=1}^T M(P^t,h_t)-\frac{2\eps}{3}
\label{eqn:randomized-boosting-eq1}
\end{align}

\prettyref{lem:avg-robust-loss-upper-bound} provides that in each iteration $t$, $L(P^t,h_t)\leq \OPTSMax+\eps/3$ given that\sareplace{~\prettyref{alg:weighted-FMS}}{~\prettyref{alg:FMS}} is 
executed for $T'=\frac{36\ln k}{\eps^2}$ rounds.
Thus, at each iteration $t$, $M(P^t,h_t)\geq 1-(\OPTSMax+\eps/3)$. Therefore, $\frac{1}{T}\sum_{t=1}^T M(P^t,h_t)\geq 1-(\OPTSMax+\eps/3)$; combining with~\prettyref{eqn:randomized-boosting-eq1} implies that:
\begin{align*}
&\frac{1}{T}\sum_{t=1}^T M(j,h_t)\geq \frac{1}{T}\sum_{t=1}^T M(P^t,h_t)-\frac{2\eps}{3}\\
&\geq 1-(\OPT^S_{\max}+\frac{\eps}{3})-\frac{2\eps}{3}=1-(\OPT^S_{\max}+\eps)
\end{align*}
Plugging in the definition of $L(P^t,h_t)$ implies that:
\[\frac{1}{T}\sum_{t=1}^T L(j,h_t)\leq \OPT^S_{\max}+\eps\]
Which concludes the proof.
\end{proof}

\begin{cor}
~\prettyref{thm:randomized-multi-robustness} implies that if for each example a predictor is picked uniformly at random from $\calH'$ to predict its label, then for each group $G_j\in \calG$, the expected robust loss is at most $\OPT^S_{\max}+\eps$.
\label{cor:interpret-avg-loss}
\end{cor}

\begin{thm}
When $T=\calO(\frac{\ln g}{\eps^2})$,~\prettyref{alg:boosting} computes a set of hypotheses $\calH'=\{h_1,\dots,h_T\}$ such that for each group $G_j$, $\RLoss_j(\MAJ(h_1,\cdots,h_T))\leq 2(\OPT^S_{\max}+\eps)$.
\label{thm:deterministic-multi-robustness}
\end{thm}

\begin{proof}
By~\prettyref{thm:randomized-multi-robustness}, after $T=\calO(\frac{\ln g}{\eps^2})$ rounds, for each group $G_j$, $\frac{1}{T}\sum_{t=1}^T \RLoss_j(h_t)\leq \OPT^S_{\max}+\eps$. Therefore, %
the total number of robustness mistakes on $G_j$ across all the classifiers %
$h_1,\cdots,h_T$ is at most $T(\OPT^S_{\max}+\eps)|G_j|$ which is equal to $T/2\cdot 2(\OPT^S_{\max}+\eps)|G_j|$.

Therefore, the fraction of examples in $G_j$ that at least $T/2$ of the classifiers in $h_1,\cdots h_T$ make a robustness mistake on is at most $2(\OPT^S_{\max}+\eps)$. %
Hence, the fraction of examples in $G_j$ that are not robustly classified by the majority-vote classifier is at most $2(\OPT^S_{\max}+\eps)$.%
\end{proof}

\subsection{Generalization Guarantees}
\label{sec:generalization-guarantees}
In this section, we derive generalization guarantees for multi-robustness. First,~\prettyref{lem:vc-robustloss-groups} shows how to bound the VC-Dimension of the intersection of robust loss and groups. We can then invoke this Lemma to get uniform convergence guarantees that will allow us to get concentration for the conditional robust loss across groups (see \prettyref{def:multi-robustness}).

\begin{lem} [VC Dimension of Intersection of Robust Loss and Groups]
\label{lem:vc-robustloss-groups}
For any class $\calH$, any perturbation set $\calU$, and any group class $\calG$, denote the intersection function class by
\begin{align*}
&\calF^\calU_{\calH,\calG} \triangleq \{ (x,y)\mapsto \max_{z\in\calU(x)} \ind\insquare{h(z)\neq y} \wedge \ind[x\in G_j]:\\ 
&h\in\calH, G_j\in\calG \}
\end{align*}
Then, it holds that $\vc(\calF^\calU_{\calH,\calG}) \leq \Tilde{O}\inparen{\vc(\calL^{\calU}_{\calH}) + \vc(\calG)}$.
\end{lem}

\begin{thm}[Generalization guarantees for average multi-robustness]
\label{thm:generalization-multi-groups}
With $T=\calO(\ln g/\varepsilon^2)$ and $m= \Tilde{O}\inparen{\frac{\vc(\calH)\ln^2(k)}{\eps^4}+\frac{\vc(\calG) + \ln(1/\delta)}{\varepsilon^2}}$,
~\prettyref{alg:boosting} computes a set of hypotheses $\calH'=\{h_1,\dots, h_T\}$, such that $\forall G_j\in \calG$, 
\small{\begin{align*}
&\frac{1}{T}\sum_{t=1}^T\Prob_{(x,y)\in \calD}\Big[\exists z\in \calU(x): h_t(z)\neq y \mid x\in G_j \Big]\\
&\leq\inparen{1 + \frac{\varepsilon}{\Prob_{\calD}(x\in G_j)}}\inparen{\OPTSMax + \varepsilon}+\frac{\varepsilon}{\Prob_{\calD}(x\in G_j)}
\end{align*}}
\end{thm}

\begin{thm}[Generalization guarantees for $\beta$-multi-robustness]
\label{thm:generalization-multi-groups-deterministic}
With $T=\calO(\ln g/\varepsilon^2)$, \\$m=\Tilde{O}\inparen{\frac{\vc(\calH)\ln(g)\ln^2(k)}{\varepsilon^6}+\frac{\vc(\calG) + \ln(1/\delta)}{\varepsilon^2}}$, and $\beta=2$,
~\prettyref{alg:boosting} computes a set of hypotheses $\calH'=\{h_1,\dots, h_T\}$, such that $\forall G_j\in \calG$, 
\small{\begin{align*}
&\Prob_{(x,y)\in \calD}\Big[\exists z\in \calU(x): \MAJ(h_1,\dots,h_T)(z)\neq y \mid x\in G_j \Big]\\
&\leq\inparen{1 + \frac{\varepsilon}{\Prob_{\calD}(x\in G_j)}}\inparen{\beta(\OPTSMax + \varepsilon)}+\frac{\varepsilon}{\Prob_{\calD}(x\in G_j)}
\end{align*}}
\end{thm}

\begin{rem}
In~\prettyref{sec:proof-generalization-guarantees-deterministic}, we show how to achieve generalization guarantees in terms of $\OPTDMax$ instead of $\OPTSMax$.
\end{rem}

\section{Conclusion}
We exhibited an example showing how %
using $\ERM$ on an augmented dataset to learn a robust classifier can fail when the examples are robustly un-realizable. Next, we provided a ``boosting-style'' algorithm that uses $\ERM$ and obtains strong robust learning guarantees in the non-realizable regime. 
This work provides theoretical evidence that our existing methods of learning accurate classifiers i.e. \ERM, can be modified effectively to learn robust classifiers even in the agnostic robust regime. Next, we introduced a new multi-robustness objective to obtain robustness guarantees simultaneously across a collection of subgroups. We showed this objective can be achieved by adding a second layer of boosting to the first algorithm. %

Adversarial examples exist for many types of classifiers but are especially salient with modern neural-based vision methods. 
However, due to the large capacity of these networks, it is not clear that they would benefit from boosting. Therefore, the fact that our algorithms rely on boosting should not be interpreted as a firm recommendation 
to use boosting with neural networks, but instead as a theoretical proof-of-concept that plain ERM can be used to learn robust models, given the right algorithmic scheme, especially if such a scheme can reduce the effective number of perturbations available to the adversary.

\subsection*{Acknowledgements}
This work was supported in part by the National Science Foundation under grants CCF-2212968 and ECCS-2216899, by the Simons Foundation under the Simons Collaboration on the Theory of Algorithmic Fairness, and by the Defense Advanced Research Projects Agency under cooperative agreement HR00112020003. The views expressed in this work do not necessarily reflect the position or the policy of the Government and no official endorsement should be inferred. Approved for public release; distribution is unlimited. This work was done when OM was a PhD student at the Toyota Technological Institute at Chicago.
\bibliography{ref}
\bibliographystyle{plainnat}
\clearpage
\appendix
\onecolumn
\appendix

\section{Supplementary Materials}
\subsection{Proof of Lemma~\ref{lem:opt}}
\begin{proof}
By definition of $\OPT_{S,\calH}$, it follows that 
\begin{align*}
    &\OPT_{S,\calH} = \min_{h\in \calH} \frac{1}{m}\sum_{i=1}^{m} \max_{z_i\in \calU(x_i)} \ind\insquare{h(z_i)\neq y_i}\\
    &\geq \min_{h\in \calH} \max_{z_1\in \calU(x_1),\dots,z_m\in\calU(x_m)}\frac{1}{m}\sum_{i=1}^{m} \ind\insquare{h(z_i)\neq y_i}\\
    &\geq \min_{Q\in \Delta(H)} \max_{z_1\in \calU(x_1),\dots,z_m\in\calU(x_m)}\frac{1}{m}\sum_{i=1}^{m} \Ex_{h\sim Q}\ind\insquare{h(z_i)\neq y_i}\\
    &\geq  \min_{Q\in \Delta(\calH)} \max_{\substack{P_{1}\in \Delta(\calU(x_1)),\\ \dots, \\P_{m} \in \Delta(\calU(x_m))}} \frac{1}{m} \sum_{i=1}^{m} \Ex_{z_i\sim P_i } \Ex_{h\sim Q} \ind\insquare{h(z_i)\neq y_i}.
\end{align*}
\end{proof}

\subsection{Proof of \prettyref{lem:FMS}}
\begin{proof}
By the minimax theorem and \citep*[][Equation 3 and 9 in proof of Theorem 1]{DBLP:conf/colt/FeigeMS15}, we have that 
\begin{align*}
&\max_{\substack{P_1\in \Delta(\calU(x_1)),\\ \dots,\\ P_m\in \Delta(\calU(x_m))}} \sum_{i=1}^{m} \Ex_{z_i\sim P_i} \frac{1}{T} \sum_{t=1}^{T} \ind\insquare{h_t(z_i)\neq y_i} \leq\\
&\min_{Q\in \Delta(\calH)} \max_{\substack{P_{1}\in \Delta(\calU(x_1)),\\ \dots, \\P_{m} \in \Delta(\calU(x_m))}} \Ex_{z_i\sim P_i } \Ex_{h\sim Q} \ind\insquare{h(z_i)\neq y_i}\\
&+ 2\frac{\sqrt{\mathcal{L}^*m\ln k}}{T},
\end{align*}
where $\calL^*=\sum_{i=1}^{m} \max_{z\in\calU(x_i)} \sum_{t=1}^{T}\ind\insquare{h_t(z)\neq y}$. By observing that $\calL^*\leq m T$ and dividing both sides of the inequality above by $m$, we arrive at the inequality stated in the lemma.
\end{proof}

\subsection{Proof of~\prettyref{thm:generalization-FMS}}
\label{sec:proof-thm-generalization-FMS}
\begin{proof}[Proof of~\prettyref{thm:generalization-FMS}]
Let $S\sim \calD^m$ be an iid sample from $\calD$, where the size of the sample $m$ will be determined later. By invoking \prettyref{lem:FMS} and \prettyref{lem:opt}, we observe that running \prettyref{alg:FMS} on $S$ {with corresponding weights $p_1,\cdots,p_m=1/m$} for $T$ rounds, produces $h_1,\dots, h_{T}$ satisfying

{\small
\[\max_{\substack{P_1\in \Delta(\calU(x_1)),\\ \dots,\\ P_m\in \Delta(\calU(x_m))}} \frac{1}{m}\sum_{i=1}^{m} \Ex_{z_i\sim P_i} \frac{1}{T} \sum_{t=1}^{T} \ind\insquare{h_t(z_i)\neq y_i} \leq \OPT_{S,\calH} + \frac{\varepsilon}{4}
\]
}
Next, the average robust loss for the majority-vote predictor $\MAJ(h_1,\dots, h_T)$ can be bounded from above as follows:
\begin{align*}
    &\frac{1}{m} \sum_{i=1}^{m} \max_{z_i\in \calU(x_i)} \ind\insquare{\MAJ(h_1,\dots,h_T)(z_i)\neq y_i}\\
    &\leq \frac{1}{m} \sum_{i=1}^{m} \max_{z_i\in \calU(x_i)}  2 \Ex_{t\sim [T]}\ind\insquare{h_t(z_i)\neq y_i}\\
    &= 2 \frac{1}{m} \sum_{i=1}^{m} \max_{z_i\in \calU(x_i)} \frac{1}{T} \sum_{t=1}^{T} \ind\insquare{h_t(z_i)\neq y_i}\\
    &\leq 2 \max_{\substack{P_1\in \Delta(\calU(x_1)),\\ \dots,\\ P_m\in \Delta(\calU(x_m))}} \frac{1}{m}\sum_{i=1}^{m} \Ex_{z_i\sim P_i} \frac{1}{T} \sum_{t=1}^{T} \ind\insquare{h_t(z_i)\neq y_i}\\
    &\leq 2 \OPT_{S,\calH} + \frac\varepsilon2.
\end{align*}
In the second line above, the factor $2$ shows up since for any arbitrary example (z,y), if at least half the predictors make a mistake then the majority-vote is wrong, and otherwise it is correct. The factor 2 is used as a correction so that RHS is bigger than LHS, where the edge case is exactly when half the predictors make a mistake.

Next, we invoke \prettyref{lem:unif-robloss} to obtain a uniform convergence guarantee on the robust loss. In particular, we apply \prettyref{lem:unif-robloss} on the \emph{convex-hull} of $\calH$: $\calH^{T} = \{\MAJ(h_1,\dots, h_T): h_1,\dots, h_T\in \calH\}$. By a classic result due to \citet{blumer:89}, it holds that $\vc(\calH^T)=O(\vc(\calH)T\ln T)$. Combining this with \prettyref{lem:unif-robloss} and plugging-in the value of $T= \frac{32 \ln k}{\varepsilon^2}$, we get that the VC dimension of the robust loss class of $\calH^T$ is bounded from above by
\[\vc(\calL_{\calH^T}^\calU) \leq O\inparen{\frac{\vc(\calH)(\ln k)^2}{\varepsilon^2}\ln\inparen{\frac{\ln k}{\varepsilon^2}}}.\]
Finally, using Vapnik's ``General Learning'' uniform convergence \citep{vapnik:82}, with probability at least $1-\delta$ over $S\sim \calD^m$ where $m =  O\inparen{\frac{\vc(\calH)(\ln k)^2}{\varepsilon^4}\ln \inparen{\frac{\ln k}{\varepsilon^2}}+\frac{\ln(1/\delta)}{\varepsilon^2}}$, it holds that
\begin{align*}
&\forall f\in \calH^T: \Ex_{(x,y)\sim \calD} \insquare{\max_{z\in\calU(x)}\ind\insquare{f(z)\neq y}}\\
&\leq \frac{1}{m}\sum_{i=1}^{m} \max_{z_i\in\calU(x_i)}\ind\insquare{f(z_i)\neq y_i} + \frac\varepsilon4
\end{align*}
This also applies to the particular output $\MAJ(h_1,\dots, h_T)$ of \prettyref{alg:FMS}, and thus
\begin{align*}
    &\Ex_{(x,y)\sim \calD} \insquare{ \max_{z\in \calU(x)} \ind\insquare{\MAJ(h_1,\dots, h_{T(\varepsilon)})(z)\neq y} }\\ 
    &\leq \frac{1}{m} \sum_{i=1}^{m} \max_{z_i\in \calU(x_i)} \ind\insquare{\MAJ(h_1,\dots,h_T)(z_i)\neq y_i} + \frac{\varepsilon}{4}\\
    &\leq 2\OPT_{S,\calH} + \frac\varepsilon2 + \frac\varepsilon4.
\end{align*}

Finally, by applying a standard Chernoff-Hoeffding concentration inequality, we get that $\OPT_{S,\calH} \leq \OPT_\calH + \frac\varepsilon8$. Combining this with the above inequality concludes the proof.
\end{proof}
\subsection{Proof of Lemma~\ref{lem:extension-FMS-weights}}
\begin{proof}
We generalize the argument in \citet*{DBLP:conf/colt/FeigeMS15} to accommodate the weights on the samples $p_1,\dots,p_m$. %
Specifically, let
\[L^{ON}_T = \sum_{t=1}^{T}\sum_{i=1}^{m}\sum_{z \in \calU(x_i)} p_iP^{t}(z,(x_i,y_i)) \ind\insquare{h_{t}(z)\neq y_i}\]
be the loss of \prettyref{alg:FMS} after $T$ rounds, and let 
\[L^*=\max_{P} \sum_{t=1}^{T}\sum_{i=1}^{m}\sum_{z \in \calU(x_i)} p_iP(z,(x_i,y_i)) \ind\insquare{h_{t}(z)\neq y_i}\]
be the benchmark loss. We show that $L^*(1-\eta)-\frac{\ln k}{\eta}\leq L^{ON}_T$. %

To this end, define $W^t_i=\inparen{\sum_{z\in\calU(x_i)}w_t(z,(x_i,y_i))}^{p_i}$ and $W^t=\prod_{i=1}^{m}W^t_i$. Let
\begin{align*}
&F^t_{i} = p_i \cdot \frac{\sum_{z\in\calU(x)} w_t(z,(x,y))\ind\insquare{h_t(z)\neq y}}{\sum_{z\in\calU(x)} w_t(z,(x,y))} \\
&= p_i \sum_{z\in \calU(x_i)} P^{t}(z,(x_i,y_i))\ind\insquare{h_t(z)\neq y}
\end{align*}
be the loss of \prettyref{alg:FMS} on example $(x_i,y_i)$ at round $t$. Observe that by the Step \sareplace{6}{7} in \sareplace{\prettyref{alg:weighted-FMS}}{\prettyref{alg:FMS}}, it holds that $W^T_i\geq (1+\eta)^{p_i \max_{z\in \calU(x_i)}\sum_{t=1}^{T} \insquare{h_t(z)\neq y}}$, and therefore $W^T\geq (1+\eta)^{L^*}$.

Observe also
{\small
\begin{align*}
&W^{t+1}_{i} = \\
&\inparen{\sum_{z:\insquare{h_t(z)\neq y}=0} w_t(z,(x,y)) + \sum_{z:\insquare{h_t(z)\neq y}=1} (1+\eta)w_{t}(z,(x,y)) }^{p_i} \\
&= W^t_i\inparen{1+\eta \frac{F^t_{i}}{p_i}}^{p_i}
\end{align*}
}

This implies that
\begin{align*}
&W^T=\prod_{i=1}^{m} W^T_i = \prod_{i=1}^{m} \insquare{k \prod_{t=1}^{T} \inparen{1+\eta \frac{F^t_{i}}{p_i}}}^{p_i} = \\
&k^{\sum_{i=1}^{m}p_i} \prod_{i=1}^{m}\prod_{t=1}^{T} \inparen{1+\eta \frac{F^t_{i}}{p_i}}^{p_i}
\end{align*}

Combining the above we have,
\[(1+\eta)^{L^*} \leq k\prod_{i=1}^{m}\prod_{t=1}^{T} \inparen{1+\eta \frac{F^t_{i}}{p_i}}^{p_i}.\]
We then apply a logarithmic transformation on both sides
\[L^{*}\ln(1+\eta) \leq \ln k + \sum_{i=1}^{m}\sum_{t=1}^{T} p_i\ln\inparen{1+\eta \frac{F^t_{i}}{p_i}}.\]
Since $a-a^2\leq \ln(1 + a) \leq a$ for $a\geq 0$, we have
\[L^*(\eta-\eta^2)\leq \ln k + \sum_{i=1}^{m}\sum_{t=1}^{T} \eta F^t_i = \ln k+ \eta L^{ON}_{T}.\]
By dividing by $\eta$ and rearranging terms we get $L^*(1-\eta)-\frac{\ln k}{\eta}\leq L^{ON}_T$. 

By setting $\eta=\sqrt{\frac{\ln k}{L^*}}$ and observing that $L^*\leq T$, the remainder of the analysis follows similar to \citet*[][Equation 3-10 in proof of Theorem 1]{DBLP:conf/colt/FeigeMS15}.
\end{proof}

\subsection{Proof of Lemma~\ref{lem:unif-robloss}}
\begin{proof}
By finiteness of $\calU$, observe that 
for any dataset $S\in (\calX\times \calY)^m$, each robust loss vector in the set of robust loss behaviors:
$$ \Pi_{\calL^{\calU}_{\calH}}(S) = \{(f(x_1,y_1),\dots, f(x_m,y_m)): f \in \calL^{\calU}_{\calH}\}$$ maps to a 0-1 loss vector on the \emph{inflated set} 
$S_\calU=\{(z^1_1,y_1),\dots, (z^k_1,y_1), \dots, (z_m^1,y_m),\dots, (z_m^k,y_m)\}$, 
{\small
\[\Pi_{{\calH}}(S_\calU)= \{(h(z^1_1),\dots, h(z^k_1),\dots, h(z_m^1),\dots, h(z_m^k)): h\in \calH\}\]
}
Therefore, it follows that $\abs{\Pi_{\calL^{\calU}_{\calH}}(S)}\leq \abs{\Pi_{{\calH}}(S_\calU)}$. Then, by applying the Sauer-Shelah lemma, it follows that $\abs{\Pi_{{\calH}}(S_\calU)} \leq O((mk)^{\vc(\calH)})$. Then, by solving for $m$ such that $O((mk)^{\vc(\calH)}) \leq 2^m$, we get that $\vc(\calL^{\calU}_{\calH})\leq O(\vc(\calH) \log(k))$.
\end{proof}

\subsection{Proof of~\prettyref{lem:avg-robust-loss-upper-bound}}

\begin{proof}
{\small
\begin{align}
&\E_{j\in [g]}[\ell^{rob}_j(h_t)]= \sum_j P_j^t(1/|G_j|)\sum_{(x,y)\in G_j} \max_{z\in \calU(x)}\ind\insquare{h_t(z)\neq y}\label{eqn:avg-robust-loss-def}\\
&=\sum_{i=1}^m p_i \cdot \max_{z\in\calU(x)}\ind\insquare{h_t(z)\neq y} \label{eqn:distribution-samples}\\
&\leq  \max_{\substack{P'_1\in \Delta(\calU(x_1)),\\ \dots,\\P'_m\in \Delta(\calU(x_m))}} \sum_{i=1}^{m} p_i\cdot \Ex_{z_i\sim P'_i} \frac{1}{T} \sum_{\tau=1}^{T} \ind\insquare{h^{\text{FMS}}_{\tau}(z_i)\neq y_i}\label{ref:replace-FMS}\\
&\leq \min_{Q\in \Delta(\calH)} \max_{\substack{P'_{1}\in \Delta(\calU(x_1)),\\ \dots, \\P'_{m} \in \Delta(\calU(x_m))}} \sum_{i=1}^{m} p_i \Ex_{z_i\sim P'_i } \Ex_{h\sim Q} \ind\insquare{h(z_i)\neq y_i} + 2\sqrt{\frac{\ln k}{T}}\label{ref:lem-weigthed-FMS}\\
&\leq \min_{h\in \calH} \max_{\substack{P'_{1}\in \Delta(\calU(x_1)),\\ \dots, \\P'_{m} \in \Delta(\calU(x_m))}} \sum_{i=1}^{m} p_i\cdot \Ex_{z_i\sim P'_i } \ind\insquare{h(z_i)\neq y_i}+ 2\sqrt{\frac{\ln k}{T}}\\
&= \min_{h\in\calH} \max_{\substack{z_1\in \calU(x_1),\\ \dots, \\z_{m} \in \calU(x_m)}} \sum_{i=1}^{m} p_i\cdot   \ind\insquare{h(z_i)\neq y_i} + 2\sqrt{\frac{\ln k}{T}}\label{ineq:pure-strategy-suffices}\\
&\leq \min_{h\in\calH} \max_{j\in[g]} (1/|G_j|)\sum_{(x,y)\in G_j}\max_{z\in\calU(x)}\ind\insquare{h(z)\neq y}+2\sqrt{\frac{\ln k}{T}}\label{ineq:relate-to-opt-max}\\
&= OPT_{\max}+2\sqrt{\frac{\ln k}{T}}
\end{align}
}

\prettyref{eqn:avg-robust-loss-def} holds by plugging in the definition of $\RLoss_j(h_t)$(\prettyref{eqn:unweighted-robust-loss}).~\prettyref{eqn:distribution-samples} holds for a distribution $p_1,\dots,p_m$ on the samples. In~\Cref{ref:replace-FMS}, $h_t$ is replaced with the hypothesis selected by\sareplace{~\prettyref{alg:weighted-FMS}}{~\prettyref{alg:FMS}} in each round $t$.~\Cref{ref:lem-weigthed-FMS} holds by~\prettyref{lem:extension-FMS-weights}.~\Cref{ineq:pure-strategy-suffices} holds since it suffices for the max-player to pick a pure strategy.~\Cref{ineq:relate-to-opt-max} holds since the whole probability mass is put as a uniform distribution on the worst-off group. Note that when defining $p_1,\cdots,p_m$, all individuals that belong to the same group have equal weights. 
\end{proof}

\subsection{Proof of~\prettyref{cor:interpret-avg-loss}}
\begin{proof}
Expected robust loss on each group $G_j\in \calG$ is:
\begin{align}
&\frac{1}{|G_j|}\sum_{(x,y)\in G_j} \max_{z\in\calU(x)}\frac{1}{T}\sum_{t=1}^T \ind[h_t(z)\neq y]\\
=&\frac{1}{|G_j|}\sum_{(x,y)\in G_j} \max_{z\in\calU(x)}\Ex_{h_t\sim U(\calH')} \ind[h_t(z)\neq y]\\
\leq&\frac{1}{|G_j|}\sum_{(x,y)\in G_j} \Ex_{h_t\sim U(\calH')} \max_{z\in \calU(x)}\ind[h_t(z)\neq y]\label{eqn:Jensen-avg-robustness}\\
=&\frac{1}{T}\sum_{t=1}^T \frac{1}{|G_j|}\sum_{(x,y)\in G_j} \max_{z\in\calU(x)}\ind[h_t(z)\neq y]\\
=&\frac{1}{T}\sum_{t=1}^T \RLoss_j(h_t)\leq \OPTSMax+\eps \label{eqn:thm-avg-robust-loss}
\end{align}
Where~\prettyref{eqn:Jensen-avg-robustness} holds by Jensen's inequality and~\prettyref{eqn:thm-avg-robust-loss} holds by~\prettyref{thm:randomized-multi-robustness}.
\end{proof}

\subsection{Reduction from overlapping groups to disjoint groups}
\label{sec:overlapping-groups-reduction}
When the groups are overlapping, we reduce it to the case of disjoint groups. The reduction is as follows: for an input instance $\calI(\calG=\{G_1,\dots,G_g\}, S)$ of overlapping groups, create a new instance $\calI'(\calG'=\{G'_1,\dots,G'_g\},S')$ as follows. Initially, for all $G'_j\in \calG'$, $G'_j$ is an empty set. For each example $(x_i,y_i)\in S$ that belongs to a set of groups $\calG_i=\{G_{i,1},\cdots, G_{i,|\calG_i|}\}\subseteq \calG$ in $\calI$, create identical copies of $(x_i,y_i)$ and assign each copy including the original example to exactly one of the groups in $\calG'_i=\{G'_{i,1},\cdots, G'_{i,|\calG'_i|}\}$. Now we have an instance $\calI'$ with disjoint groups. By executing~\prettyref{alg:boosting} on $\calI'$, it returns a predictor $h$ that achieves %
a $\beta$-multi-robustness guarantee. First, we argue that if $h$ is used on $\calI$, it achieves a multi-robustness guarantee of $\beta\cdot \OPT^{\calI'}_{\max}$. This is the case since either $h$ makes a robustness mistake on all copies of an example or does not make any robustness mistakes on any of them. Next, we show that $\OPT^{\calI'}_{\max}\leq \OPT^{\calI}_{\max}$. Consider a predictor $h^*\in \calH$ that achieves multi-robustness of $\OPT^{\calI}_{\max}$ on $\calI$. If $h^*$ is used on $\calI'$, for each example $(x,y)\in S$ that $h^*$ has zero robust loss on, it does not make any mistakes on any of its copies in $\calI'$. Additionally, if $h^*$ makes a robustness mistake on $(x,y)$, then it makes a robustness mistake on all its copies in $\calI'$. Thus, $h^*$ achieves a multi-robustness guarantee of $\OPT^{\calI}_{\max}$ on $\calI'$. Therefore, $\OPT^{\calI'}_{\max}\leq \OPT^{\calI}_{\max}$, and a $\beta\cdot \OPT^{\calI'}_{\max}$ multi-robustness guarantee on $\calI$ implies $\beta\cdot \OPT^{\calI}_{\max}$ multi-robustness. A similar argument holds for the average multi-robustness guarantee.

\begin{rem}
When $|\calG|$ is large, this reduction becomes computationally inefficient, since in the worst case, the number of samples gets increased by a multiplicative factor of $|\calG|$. However, this reduction is equivalent to keeping only one copy of each sample $(x_i,y_i)\in S$ and when executing~\prettyref{alg:boosting}, in each iteration $t$, assigning it a weight of $p_i=\sum_{j\in[g]:(x_i,y_i)\in G_j}P^t_j/|G_j|$.
\end{rem}
\subsection{Proof of~\prettyref{lem:vc-robustloss-groups}}
\begin{proof}
The proof is inspired by the proof of \citep[claim B.1 in ][]{DBLP:conf/icml/KearnsNRW18} which proved a similar result for the standard $0$-$1$ loss, and here we extend the result to the robust loss using essentially the same proof. 

Let $S\subseteq \calX \times \calY$ be a dataset of size $m$ that is shattered by $\calF^\calU_{\calH,\calG}$. Then, observe that, by definition of $\calF^\calU_{\calH,\calG}$, the number of possible behaviors $\abs{\Pi_{\calF^\calU_{\calH,\calG}}(S)}$ is at most $\abs{\Pi_{\calL^{\calU}_{\calH}}(S)}\cdot \abs{\Pi_\calG(S)}$. By Sauer-Shelah Lemma, $\abs{\Pi_{\calL^{\calU}_{\calH}}(S)} \leq O(m^{\vc(\calL^{\calU}_{\calH})})$ and $ \abs{\Pi_\calG(S)}\leq O(m^{\vc(\calG)})$. Thus, $\abs{\Pi_{\calF^\calU_{\calH,\calG}}(S)} = 2^m \leq O(m^{\vc(\calL^{\calU}_{\calH})+\vc(\calG)})$, and solving for $m$ yields that $m=\Tilde{O}(\vc(\calL^{\calU}_{\calH})+\vc(\calG))$. Hence, $\vc(\calF^\calU_{\calH,\calG}) \leq \Tilde{O}\inparen{\vc(\calL^{\calU}_{\calH}) + \vc(\calG)}$.
\end{proof}

\subsection{Proof of~\prettyref{thm:generalization-multi-groups}}
\label{sec:proof-generalization-guarantees}

\begin{proof}
The output of~\prettyref{alg:boosting} is $\calH'=\{h_1,\dots,h_T\}$ where each of the predictors $h_1,\dots,h_T$ is a majority-vote predictor over $\calH$. Due to \citet{blumer:89}, the VC-dimension of the output space is $\vc(\calH^{T'})=\Big(\vc(\calH)T'\ln T'\Big)$ where $T'$ is the number of rounds of\sareplace{~\prettyref{alg:weighted-FMS}}{~\prettyref{alg:FMS}} in each oracle call.

Set $m= \Tilde{O}\inparen{\frac{\vc(\calH^{T'})\ln (k)+\vc(\calG) + \ln (1/\delta)}{\varepsilon^2}}$.  
By setting $T'=\calO(\frac{\ln k}{\eps^2})$ and by invoking \prettyref{lem:unif-robloss} and \prettyref{lem:vc-robustloss-groups} on the hypothesis class $\calH$ and group class $\calG$, we get the following uniform convergence guarantee. With probability at least $1-\delta$ over $S\sim \calD^m$,

\begin{align*}
&\inparen{\forall h\in \calH^{T'}}\inparen{\forall G_j\in\calG}:\\
&\Bigg\lvert\Ex_{(x,y)\sim \calD} \insquare{\ind[x\in G_j] \wedge \max_{z\in \calU(x)}\ind[h(z)\neq y] } - \frac{1}{m}\sum_{(x,y)\in S} \ind[x\in G_j] \wedge \max_{z\in \calU(x)}\ind[h(z)\neq y]\Bigg\rvert \leq \varepsilon
\end{align*}

We can rewrite the above guarantee in a conditional form which will be useful for us shortly in the proof. Namely, $\forall h\in \calH^{T'}, \forall G_j\in\calG$:
\begin{align}
    &\Prob_{(x,y)\sim \calD}\insquare{\exists z\in \calU(x): h(z)\neq y | x\in G_j } 
    \leq \frac{\Prob_{S}(x\in G_j)}{\Prob_{\calD}(x\in G_j)} \frac{1}{|G_j|} \sum_{(x,y)\in S\wedge x\in G_j}\max_{z\in \calU(x)}\ind[h(z)\neq y]
    + \frac{\varepsilon}{\Prob_{D}(x\in G_j)}\label{eqn:conditional-uniform-convergence}
\end{align}
where $|G_j|=\sum_{(x,y)\in S} \ind[x\in G_j]$. 

\prettyref{thm:avg-empirical-boosting} shows that running \prettyref{alg:boosting} produces hypotheses $h_1,\dots, h_T$ such that, $\forall G_j\in \calG$:
\begin{align}
   &\frac{1}{T}\sum_{t=1}^{T} \frac{1}{|G_j|}\sum_{(x,y)\in S\wedge x\in G_j} \max_{z\in \calU(x)}\ind[h_t(z)\neq y] \leq \OPTSMax + \varepsilon\label{eqn:thm-alg-boosting-avg}
\end{align}

\prettyref{eqn:conditional-uniform-convergence} implies that $\forall G_j\in\calG$,
{\small\begin{align}
    \frac{1}{T}\sum_{t=1}^T\Prob_{(x,y)\sim \calD}\insquare{\exists z\in \calU(x): h_t(z)\neq y | x\in G_j } \leq 
    \frac{1}{T}\sum_{t=1}^T\frac{\Prob_{S}(x\in G_j)}{\Prob_{\calD}(x\in G_j)} \frac{1}{|G_j|} \sum_{(x,y)\in S\wedge x\in G_j}\max_{z\in \calU(x)}\ind[h_t(z)\neq y]
    +\frac{\varepsilon}{\Prob_{\calD}(x\in G_j)},\label{eqn:avg-uniform-convergence-conditional}
\end{align}}

Combining~\prettyref{eqn:thm-alg-boosting-avg} and~\prettyref{eqn:avg-uniform-convergence-conditional} implies:

\begin{align}
&\frac{1}{T}\sum_{t=1}^T\Prob_{(x,y)\in \calD}\insquare{\exists z\in \calU(x): h_t(z)\neq y | x\in G_j } 
\leq\frac{\Prob_{S}(x\in G_j)}{\Prob_{\calD}(x\in G_j)} \inparen{\OPTSMax + \varepsilon}+\frac{\varepsilon}{\Prob_{\calD}(x\in G_j)}
\label{eqn:avg-generalization-prob-sample}
\end{align}

Now, given additional samples $\tilde{m}= O\inparen{\frac{\vc(\calG)+\log(2/\delta)}{\varepsilon^2}}$, in addition to the above, we can guarantee that:
\begin{align}
\forall G_j\in \calG: \frac{\Prob_{S}(x\in G_j)}{\Prob_{\calD}(x\in G_j)}
 \leq \frac{\Prob_{\calD}(x\in G_j) + \varepsilon}{\Prob_{\calD}(x\in G_j)} = 1 + \frac{\varepsilon}{\Prob_{\calD}(x\in G_j)}.
\label{eqn:avg-generalization-prob-distribution}
\end{align}

Combining~\prettyref{eqn:avg-generalization-prob-sample} and~\prettyref{eqn:avg-generalization-prob-distribution} implies that:

\begin{align*}
\frac{1}{T}\sum_{t=1}^T\Prob_{(x,y)\sim \calD}\insquare{\exists z\in \calU(x): h_t(z)\neq y | x\in G_j } 
\leq \inparen{1 + \frac{\varepsilon}{\Prob_{\calD}(x\in G_j)}}\inparen{\OPTSMax + \varepsilon}+\frac{\varepsilon}{\Prob_{\calD}(x\in G_j)}
\end{align*}

which completes the proof. We can also obtain a bound in terms of $\OPTDMax$ instead of $\OPTSMax$ using a similar approach used in~\prettyref{sec:proof-generalization-guarantees-deterministic}.

\end{proof}

\subsection{Proof of~\prettyref{thm:generalization-multi-groups-deterministic}}
\label{sec:proof-generalization-guarantees-deterministic}

\begin{proof}
The output of~\prettyref{alg:boosting} is $\calH'=\{h_1,\dots,h_T\}$. Taking majority-vote over the predictors in $\calH'$ is equivalent to taking the majority-vote of majority-vote predictors over $\calH$. Therefore, due to \citet{blumer:89}, the VC-dimension of the output space is $\vc(\calH^{T'})^T=\Big(\vc(\calH)T'\ln T'\Big)T\ln T$, where $T'$ is the number of rounds of\sareplace{~\prettyref{alg:weighted-FMS}}{~\prettyref{alg:FMS}} in each oracle call and $T$ is the number of rounds of~\prettyref{alg:boosting}.

Let the sample size $m= \Tilde{O}\inparen{\frac{\vc(\calH^{T'})^T\log(k)+\vc(\calG) + \log(1/\delta)}{\varepsilon^2}}$. By setting %
$T=\calO(\ln g/\eps^2)$ and $T'=\calO(\frac{\ln k}{\eps^2})$ and by invoking \prettyref{lem:unif-robloss} and \prettyref{lem:vc-robustloss-groups} on the hypothesis class $\calH$ and group class $\calG$, we get the following uniform convergence guarantee. With probability at least $1-\delta$ over the sample set $S\sim \calD^m$, $\forall h\in (\calH^{T'})^T$ and $\forall G_j\in\calG$:
\begin{align}
    &\Bigg\lvert\Ex_{(x,y)\sim \calD} \insquare{\ind[x\in G_j] \wedge \max_{z\in \calU(x)}\ind[h(z)\neq y] } 
    -\frac{1}{m}\sum_{(x,y)\in S} \ind[x\in G_j] \wedge \max_{z\in \calU(x)}\ind[h(z)\neq y]\Bigg\rvert
    \leq \varepsilon\label{eqn:uniform-convergence-deterministic-eq1}
\end{align}

We can rewrite the above guarantee in a conditional form which will be useful for us shortly in the proof. Namely, $\forall h\in (\calH^{T'})^T$ and $\forall G_j\in\calG$:
\begin{align}
&\Prob_{(x,y)\sim \calD}\insquare{\exists z\in \calU(x): h(z)\neq y | x\in G_j }
\leq \frac{\Prob_{S}(x\in G_j)}{\Prob_{\calD}(x\in G_j)} \frac{1}{|G_j|} \sum_{(x,y)\in S\wedge x\in G_j}\max_{z\in \calU(x)}\ind[h(z)\neq y] 
+ \frac{\varepsilon}{\Prob_{\calD}(x\in G_j)}\label{eqn:uniform-convergence-deterministic-proof}
\end{align}
where $|G_j|=\sum_{(x,y)\in S} \ind[x\in G_j]$.
\prettyref{thm:deterministic-multi-robustness} provides that %
$h^{\text{maj}}=\MAJ(h_1,\dots,h_T)$ satisfies that $\forall G_j\in \calG$:
\begin{align}
   &\frac{1}{|G_j|}\sum_{(x,y)\in S\wedge x\in G_j} \max_{z\in \calU(x)}\ind[h^{\text{maj}}(z)\neq y] \leq \beta (\OPTSMax + \varepsilon)\label{eqn:guarantee-deterministic-eq1}
\end{align}
Combining~\prettyref{eqn:uniform-convergence-deterministic-proof} and~\prettyref{eqn:guarantee-deterministic-eq1} implies that $\forall G_j\in \calG$:
\begin{align}
&\Prob_{(x,y)\sim \calD}\insquare{\exists z\in \calU(x): h^{\text{maj}}(z)\neq y | x\in G_j }
\leq \frac{\Prob_{S}(x\in G_j)}{\Prob_{\calD}(x\in G_j)} \inparen{\beta(\OPTSMax + \varepsilon)}+\frac{\varepsilon}{\Prob_{\calD}(x\in G_j)}
\label{eqn:deterministic-robustness-eq2}
\end{align}

Now, given additional samples $\tilde{m}= O\inparen{\frac{\vc(\calG)+\log(2/\delta)}{\varepsilon^2}}$, guarantees that:
\begin{align} 
\forall G_j\in \calG: \frac{\Prob_{S}(x\in G_j)}{\Prob_{\calD}(x\in G_j)} \leq \frac{\Prob_{\calD}(x\in G_j) + \varepsilon}{\Prob_{\calD}(x\in G_j)} = 1 + \frac{\varepsilon}{\Prob_{\calD}(x\in G_j)}
\label{eqn:deterministic-robustness-eq3}
\end{align}

Combining~\prettyref{eqn:deterministic-robustness-eq2} and~\prettyref{eqn:deterministic-robustness-eq3} gives a refined bound on the average conditional robust loss that holds uniformly across groups. Namely, $\forall G_j\in \calG$,
\begin{align*}
& \Prob_{(x,y)\sim \calD}\insquare{\exists z\in \calU(x): h^{\text{maj}}(z)\neq y | x\in G_j }
\leq\inparen{1 + \frac{\varepsilon}{\Prob_{\calD}(x\in G_j)}}\inparen{\beta(\OPTSMax + \varepsilon)}+\frac{\varepsilon}{\Prob_{\calD}(x\in G_j)}
\end{align*}
We can also obtain a guarantee in terms of $\OPTDMax$ instead of $\OPTSMax$, as follows. Let $h^*\in \calH$ be a predictor which attains $\OPTDMax$ defined as 
{\small
\[\OPTDMax = \min_{h\in\calH} \max_{G_j\in\calG} \Ex_{(x,y)\sim \mathcal{D}}\insquare{\max_{z\in \calU(x)} \ind[h(z)\neq y] \bigg| x\in G_j}.\]
}
Dividing both sides of \prettyref{eqn:uniform-convergence-deterministic-eq1} by $\Prob_{S}(x\in G_j)$ provides that $\forall G_j\in\calG, \forall h\in\calH$:

\begin{align*}
&\Bigg\lvert\frac{\Prob_{\calD}(x\in G_j)}{\Prob_{S}(x\in G_j)} \Prob_{(x,y)\in \calD}\insquare{\exists z\in \calU(x): h(z)\neq y | x\in G_j } - \Prob_{(x,y)\in S}\insquare{\exists z\in \calU(x): h(z)\neq y | x\in G_j }\Bigg\rvert%
\leq \frac{\varepsilon}{\Prob_{S}(x\in G_j)}
\end{align*}

and thus it implies that %
\begin{align*}
&\Prob_{(x,y)\in S}\insquare{\exists z\in \calU(x): h(z)\neq y | x\in G_j }\leq 
\inparen{1+\frac{\varepsilon}{\Prob_{S}(x\in G_j)}}\Prob_{(x,y)\sim \calD}\insquare{\exists z\in \calU(x): h(z)\neq y | x\in G_j }
+ \frac{\varepsilon}{\Prob_{S}(x\in G_j)}
\end{align*}

Supposing that $\forall G_j\in\calG$, $\Prob_{S}(x\in G_j)\geq \gamma$. By taking a max over groups $G_j \in \calG$, we get
\[\OPTSMax \leq (1+\frac{\varepsilon}{\gamma}) \OPTDMax +\frac{\varepsilon}{\gamma}.\]

\end{proof}

\end{document}